\documentclass{article}
\usepackage{spconf,amsmath,graphicx,hyperref}

\usepackage{amssymb}
\usepackage{mathtools}
\usepackage{amsthm}
\usepackage{caption}
\usepackage{subcaption}
\usepackage{mathtools}
\usepackage{xcolor}         
\usepackage{enumitem}

\usepackage[utf8]{inputenc} 
\usepackage[T1]{fontenc}    
\usepackage{url}            
\usepackage{booktabs}       
\usepackage{amsfonts}       
\usepackage{nicefrac}       
\usepackage{cite}

\usepackage{graphicx}
\usepackage{multirow}
\newcommand{\method}{SoTPDEG}

\theoremstyle{plain}
\newtheorem{theorem}{Theorem}[section]
\newtheorem{proposition}[theorem]{Proposition}
\newtheorem{lemma}{Lemma}

\newtheorem{corollary}[theorem]{Corollary}
\theoremstyle{definition}
\newtheorem{definition}[theorem]{Definition}

\theoremstyle{remark}
\newtheorem{remark}[theorem]{Remark}

\newcommand{\ie}{\textit{i}.\textit{e}., }
\newcommand{\eg}{\textit{e}.\textit{g}., }

\DeclareMathOperator{\tr}{\operatorname{tr}}

\newcommand{\vertiii}[1]{{\left\vert\kern-0.25ex\left\vert\kern-0.25ex\left\vert #1 
    \right\vert\kern-0.25ex\right\vert\kern-0.25ex\right\vert}}

\title{Second-Order Tensorial Partial Differential Equations on Graphs}
\name{Aref Einizade$^1$, Fragkiskos D. Malliaros$^2$, Jhony H. Giraldo$^1$\thanks{This research was supported by the ANR projects DeSNAP (ANR-24-CE23-1895-01) and GraphIA (ANR-20-CE23-0009-01).}}
\address{$^1$LTCI, Télécom Paris, Institut Polytechnique de Paris, France\\
$^2$Université Paris-Saclay, CentraleSupélec, Inria, France}
\begin{document}
%
\maketitle
\begin{abstract}

Processing data on multiple interacting graphs is crucial for many applications, but existing approaches rely mostly on discrete filtering or first-order continuous models, dampening high frequencies and slow information propagation. In this paper, we introduce second-order tensorial partial differential equations on graphs (SoTPDEG) and propose the first theoretically grounded framework for second-order continuous product graph neural networks (GNNs). Our method exploits the separability of cosine kernels in Cartesian product graphs to enable efficient spectral decomposition while preserving high-frequency components. We further provide rigorous over-smoothing and stability analysis under graph perturbations, establishing a solid theoretical foundation. Experimental results on spatiotemporal traffic forecasting illustrate the superiority over the compared methods.

\end{abstract}
\section{Introduction}
\label{sec:intro}

Tensors \cite{kolda2009tensor}, which generalize matrices to higher dimensions, arise in diverse domains such as brain analysis \cite{song2023brain,einizade2023iterative}, recommender systems \cite{ortiz2019sparse}, and spatiotemporal forecasting \cite{einizade2024continuous,cini2023graph}.
When tensors are linked to multiple interdependent graphs, they give rise to \textit{multidomain graph data} \cite{kadambari2021product,monti2017geometric}. Unlike conventional methods in graph signal processing (GSP) \cite{ortega2018graph} and graph machine learning (GML), typically considering a single graph \cite{einizade2023learning}, this setting requires jointly processing across several interacting graph domains. 
Although such data structures are increasingly common in real-world applications, the methods for learning from multidomain graph data remain scarce \cite{sabbaqi2023graph,einizade2024continuous,monti2017geometric,castro2024gegenbauer}.
Thus, developing graph learning techniques for tensorial data holds great promise for advancing the theory and applications in GSP and GML.

A common strategy for learning from multidomain graph data is to use product graphs (PGs) \cite{kadambari2021product,sabbaqi2023graph,einizade2023productgraphsleepnet,einizade2023learning}, a concept rooted in GSP \cite{ortega2018graph}, where discrete filtering operations combine information across domains.
For example, the graph-time convolution neural network (GTCNN) \cite{sabbaqi2023graph} introduces a discrete filtering mechanism to jointly process spatial and temporal information, but suffers from three key limitations: \textit{i}) it requires costly grid searches to tune filter orders, \textit{ii}) polynomial filters restrict long-range receptive fields \cite{behmanesh2023tide,han2024continuous}, and \textit{iii}) it is inherently constrained to two factor graphs (space and time).
Besides, these methods inherit well-known limitations of conventional graph neural networks (GNNs), \eg over-smoothing and over-squashing \cite{cai2020note,giraldo2023trade}, restricting receptive fields and long-range modeling \cite{behmanesh2023tide}. 

In another line of research, continuous GNNs (CGNNs) have emerged as an effective approach to mitigate the challenges of over-smoothing and over-squashing \cite{han2024continuous}.
CGNNs formulate graph learning as solving partial differential equations (PDEs) on graphs—such as the heat or wave equations—through neural architectures that approximate these solutions \cite{han2024continuous,xhonneux2020continuous}. 
While these frameworks have shown promise, they are mostly designed for single-graph settings.
A notable exception is CITRUS \cite{einizade2024continuous}, which extends CGNNs to multidomain data by introducing TPDEGs and using PGs.
Despite its contributions, CITRUS is limited to first-order derivatives \cite{einizade2024continuous}, dampening high-frequency signals and slowing information propagation—issues that are particularly problematic in oscillatory (periodic or fine-grained) tasks such as spatiotemporal forecasting, heterophilic networks, and continuum mechanics simulations \cite{yue2025graph,yue2025hyperbolicpde,rusch2022graph}. 

To overcome these limitations, we propose the first principled second-order TPDEGs (\method), which provide a stronger foundation for modeling multidomain data across multiple interacting graphs. Building on this formulation, we introduce a continuous graph learning model that efficiently solves \method~using separable cosine oscillatory kernels on Cartesian PGs, preserving high-frequency components. This design also allows receptive fields to be learned adaptively during training, eliminating the costly grid searches required by discrete methods. 
Our design can alleviate the computational burden of full spectral decompositions by using a small subset of eigenvalue decompositions from the factor graphs. We further provide rigorous theoretical analyses, showing that our model preserves stability under graph perturbations and controls the over-smoothing rate. 
Importantly, the number of learnable parameters is independent of the number of factor graphs, ensuring scalability to large and complex multidomain settings, with real-world experimental validations on spatiotemporal traffic forecasting.

The key contributions of this work are threefold.
First, we propose \method~as a unified framework for modeling multidomain graph data. Then, we derive a continuous graph filtering mechanism over Cartesian product spaces as a natural solution to \method. Second, we provide comprehensive theoretical analyses of our model, establishing its stability and demonstrating its ability to alleviate over-smoothing. Lastly, we evaluate the \method~on a real-world spatiotemporal traffic forecasting task showcasing the superiority over the compared methods.

\section{Methodology}
\textbf{Preliminaries and notation.} An undirected, weighted graph $\mathcal{G}$ with $N$ vertices is represented as 
$\mathcal{G} = \{\mathcal{V}, \mathcal{E}, \mathbf{A}\}$, 
where $\mathcal{V}$ and $\mathcal{E}$ denote the sets of vertices and edges, respectively. 
The adjacency matrix $\mathbf{A} \in \mathbb{R}^{N \times N}$ characterizes the graph structure, 
with $\mathbf{A}_{ij} = \mathbf{A}_{ji} \ge 0$ indicating the connection strength between 
vertices $i$ and $j$. We assume $\mathbf{A}_{ii} = 0$ for all $i$, meaning the graph contains no self-loops. The graph Laplacian $\mathbf{L} \in \mathbb{R}^{N \times N}$ is defined as $\mathbf{L} = \mathbf{D} - \mathbf{A},$ where $\mathbf{D} = \mathrm{diag}(\mathbf{A}\mathbf{1})$ is the diagonal degree matrix and  $\mathbf{1}$ is the all-ones vector. A \emph{graph signal} is defined as a mapping $x: \mathcal{V} \to \mathbb{R}$, which assigns 
a real value to each node. Such a signal can be represented in vector form as 
$\mathbf{x} = [x_1, \ldots, x_N]^\top$. For a general vector  $\mathbf{a} = [a_1, \ldots, a_N]^\top \in \mathbb{R}^{N \times 1}$, we define the element-wise 
cosine as $\cos(\mathbf{a}) = [\cos(a_1), \ldots, \cos(a_N)]^\top.$ We denote the vectorization operator by $\mathrm{vec}(\cdot)$. Finally, using the Kronecker product $\otimes$, the Kronecker sum \cite{sandryhaila2014big} (also referred to as the Cartesian product) $\oplus$, and the Laplacian factors $\{\mathbf{L}_p \in \mathbb{R}^{N_p \times N_p}\}_{p=1}^P$, we introduce the following definitions:
\begin{equation}
\begin{split}
&\oplus_{p=1}^{P}{\mathbf{L}_p}\vcentcolon=\mathbf{L}_1\oplus\hdots\oplus\mathbf{L}_P,\,\downarrow\oplus_{p=1}^{P}{\mathbf{L}_p}\vcentcolon=\mathbf{L}_P\oplus\hdots\oplus\mathbf{L}_1,\\
&\otimes_{p=1}^{P}{\mathbf{L}_p}\vcentcolon=\mathbf{L}_1\otimes\hdots\otimes\mathbf{L}_P,\,\downarrow\otimes_{p=1}^{P}{\mathbf{L}_p}\vcentcolon=\mathbf{L}_P\otimes\hdots\otimes\mathbf{L}_1.
\end{split}
\end{equation}
Here, $\mathbf{I}_n$ denotes the identity matrix of size $n$. 
We define a $D$-dimensional tensor as 
$\underline{\mathbf{U}} \in \mathbb{R}^{N_1 \times \cdots \times N_D}$. 
The matrix $\underline{\mathbf{U}}_{(i)} \in \mathbb{R}^{N_i \times 
\left[\prod_{r=1, r \neq i}^{D} N_r\right]}$ represents the \emph{$i$-th mode matricization} 
(unfolding) of the tensor $\underline{\mathbf{U}}$ \cite{kolda2009tensor}.
The \emph{mode-$i$ tensor multiplication} of a tensor 
$\underline{\mathbf{U}}$ by a matrix $\mathbf{X} \in \mathbb{R}^{m \times N_i}$ 
is defined as $\underline{\mathbf{G}} = \underline{\mathbf{U}} \times_i \mathbf{X},$ where 
$\underline{\mathbf{G}} \in \mathbb{R}^{N_1 \times \cdots \times N_{i-1} \times 
m \times N_{i+1} \times \cdots \times N_D}$ \cite{kolda2009tensor}.
All the proofs of theorems and propositions of this paper are provided in the Appendix.


\subsection{Second-Order TPDEG}
We first propose the second-order TPDEG (\method) as:
\begin{definition}[\method]
\label{MPDEG_def}
Let $\underline{\mathbf{U}}_t \in \mathbb{R}^{N_1 \times N_2 \times \cdots \times N_P}$ 
denote a multidomain tensor whose elements vary as functions of time $t$. 
We further define the \method~with $\{\mathbf{L}_p \in \mathbb{R}^{N_p \times N_p}\}_{p=1}^P$ 
as the set of factor Laplacians, as:
\begin{equation}
\label{MDPDE}
\frac{\partial^2\underline{\mathbf{U}}_t}{\partial\:t^2} = - \sum_{i=1}^{P} \underline{\mathbf{U}}_t \times_i \mathbf{L}_i^2 - 2 \sum_{1 \le i < j \le P} \underline{\mathbf{U}}_t \times_i \mathbf{L}_i \times_j \mathbf{L}_j.
\end{equation}
\end{definition}
Here, there are two important key differences with the first-order TPDEGs \cite{einizade2024continuous}: (\textit{i}) The first term in the RHS applies $\mathbf{L}^2_i$, compared to TPDEG, which exploits $\mathbf{L}_i$.
Therefore, \method~potentially searches for further neighborhoods. 
(\textit{ii}) The second term in the RHS of \eqref{MDPDE} models the cross-interaction between factor graphs, compared to the TPDEG relying only on the domain independencies. 


\begin{figure}
    \centering
    \includegraphics[width=\columnwidth]{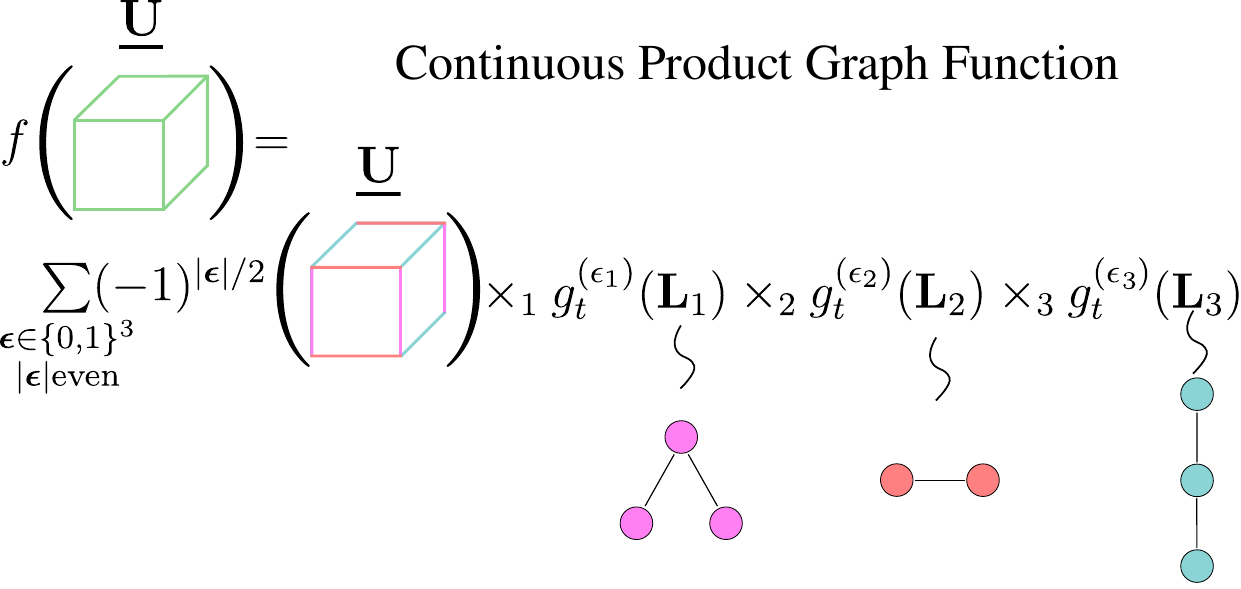}
    \caption{Continuous product graph function in \method~on the tensor $\underline{\mathbf{U}}$.}
    \label{fig:pipeline}
\end{figure}

\begin{theorem}
\label{MDPDE_thm}
Assuming \(\tilde{\underline{\mathbf{U}}}_0\) as the initial value of \(\tilde{\underline{\mathbf{U}}}_t\), the solution to the \method~in \eqref{MDPDE} (with the assumption of $\left. \frac{\partial\tilde{\underline{\mathbf{U}}}_t}{\partial t}\right|_{t=0}=0$) is given by:
{\small
\begin{equation}
\label{MDPDE_sol}
\begin{split}
&\tilde{\underline{\mathbf{U}}}_t =\sum_{\substack{\boldsymbol{\epsilon} \in \{0,1\}^P \\ |\boldsymbol{\epsilon}|\, \text{even}}} 
(-1)^{|\boldsymbol{\epsilon}|/2} \; 
\tilde{\underline{\mathbf{U}}}_0 \times_1 g_t^{(\epsilon_1)}(\mathbf{L}_1) \cdots \times_P g_t^{(\epsilon_P)}(\mathbf{L}_P),
\end{split}
\end{equation}}
where $g_t^{(0)}(\mathbf{L}_i) = \cos(t\mathbf{L}_i), \:\: g_t^{(1)}(\mathbf{L}_i) = \sin(t \mathbf{L}_i)$, $\boldsymbol{\epsilon} = (\epsilon_1, \dots, \epsilon_P) \in \{0,1\}^P$, with $\{0,1\}^P$ being the Cartesian product between $P$ sets of $\{0,1\}$. Besides, $|\boldsymbol{\epsilon}| := \sum_{i=1}^{P} \epsilon_i$.
\end{theorem}
Let $\underline{\mathbf{U}}_l\in\mathbb{R}^{N_1\times\hdots \times N_P\times F_l}$ be the input tensor with $F_l$ being the number of features, $t$ the learnable receptive fields, $\mathbf{W}_{l}\in\mathbb{R}^{F_l\times F_{l+1}}$ a matrix of learnable parameters, and $f(\underline{\mathbf{U}}_l)\in\mathbb{R}^{N_1\times\hdots \times N_P\times F_{l+1}}$ the output tensor. Using Theorem \ref{MDPDE_thm}, the core function of our framework is defined as:
{\small
\begin{equation}
\label{tensor_formm2}
\begin{split}
&f(\underline{\mathbf{U}}_l) = \\
&\sum_{\substack{\boldsymbol{\epsilon} \in \{0,1\}^P \\ |\boldsymbol{\epsilon}|\, \text{even}}} 
(-1)^{|\boldsymbol{\epsilon}|/2} \; 
\underline{\mathbf{U}}_l \times_1 g_{t}^{(\epsilon_1)}(\mathbf{L}_1) \cdots \times_P g_{t}^{(\epsilon_P)}(\mathbf{L}_P)\times_{P+1}\mathbf{W}^\top_{l}.
\end{split}
\end{equation}}
Thus, the core graph filtering operation is obtained on a PG as follows (see Figure \ref{fig:pipeline} for a high-level picture):
\begin{proposition}
\label{prop:MDPDE_theorem}
The core function of \method~in \eqref{tensor_formm2}, \ie $\underline{\mathbf{U}}_{l+1}=f(\underline{\mathbf{U}}_l)$, can be rewritten as:
\begin{equation}
\label{imp_eq}
\left[f(\underline{\mathbf{U}}_l)_{(P+1)}\right]^\top=\cos(t\mathbf{L}_\diamond){[\underline{\mathbf{U}}_l}_{(P+1)}]^\top\mathbf{W}_{l},
\end{equation}
where $\mathbf{L}_\diamond \vcentcolon= \downarrow\oplus_{p=1}^{P}{\mathbf{L}_p}$ is the Laplacian of the Cartesian PG.
\end{proposition}
Proposition \ref{prop:MDPDE_theorem} provides the theoretical basis for the implementation of our framework.
To compute the solution efficiently, we can use the spectral decompositions of the factor and product Laplacians $\{\mathbf{L}_p = \mathbf{V}_p \boldsymbol{\Lambda}_p \mathbf{V}_p^\top\}_{p=1}^P$ and $\{\mathbf{L}_\diamond = \mathbf{V}_\diamond \boldsymbol{\Lambda}_\diamond \mathbf{V}_\diamond^\top\}$, respectively. 
Here, the eigenvectors of the PG are obtained as $\mathbf{V}_\diamond = \downarrow \otimes_{p=1}^P \mathbf{V}_p$, and the corresponding eigenvalues follow $\boldsymbol{\Lambda}_\diamond = \downarrow \oplus_{p=1}^P \boldsymbol{\Lambda}_p$ \cite{sandryhaila2014big}.
For each factor graph, we choose the top $K_p \leq N_p$ eigenvalue–eigenvector pairs to reduce cost.
More precisely, it can be shown \cite{behmanesh2023tide} that Eqn.~\eqref{imp_eq} can be equivalently expressed in the following form:
\begin{equation}
\begin{split}
&\left[f(\underline{\mathbf{U}}_l)_{(P+1)}\right]^\top = \\
&\mathbf{V}^{(K_p)}_\diamond\left(\underbrace{\overbrace{[\tilde{\boldsymbol{\lambda}}_l|\hdots|\tilde{\boldsymbol{\lambda}}_l]}^{F_l\:\text{times}}}_{\tilde{\boldsymbol{\Lambda}}_l}\odot\left({\mathbf{V}^{(K_p)}_\diamond}^\top\left[{\underline{\mathbf{U}}_{l}}_{(P+1)}\right]^\top\right)\right)\mathbf{W}_{l},
\end{split}
\label{eff_imp}
\end{equation}
\begin{equation}
\label{lamb_V}
\begin{split}
\text{with}\,\,\,\,\,\,\,\,\,&\tilde{\boldsymbol{\lambda}}_l=\sum_{\substack{\boldsymbol{\epsilon} \in \{0,1\}^P\\|\boldsymbol{\epsilon}|\, \text{even}}} 
{(-1)^{|\boldsymbol{\epsilon}|/2}\downarrow\otimes_{p=1}^{P}{g_{t}^{(\epsilon_p)}(\boldsymbol{\lambda}^{(K_p)}_p)}},\\
&\mathbf{V}^{(K_p)}_\diamond=\downarrow\otimes_{p=1}^{P}{\mathbf{V}^{(K_p)}_p}.
\end{split}
\end{equation}
Here, $\boldsymbol{\lambda}^{(K_p)}_p \in \mathbb{R}^{K_p \times 1}$ and 
$\mathbf{V}^{(K_p)}_p \in \mathbb{R}^{N_p \times K_p}$ represent the top $K_p \leq N_p$ 
eigenvalues and their corresponding eigenvectors of $\mathbf{L}_p$, 
selected according to the largest magnitudes of the eigenvalues.  
The symbol $\odot$ denotes the element-wise (Hadamard) product.
Then, the output of the $(l)$-th layer in our model is expressed as  
${\underline{\mathbf{U}}_{l+1}}_{(P+1)} = 
\sigma \left( f(\underline{\mathbf{U}}_l)_{(P+1)} \right)$,  
where $\sigma(\cdot)$ denotes a suitable nonlinear activation function.
\begin{remark}
\label{remark_1}
The eigenvalue decomposition (EVD) of the PG Laplacian $\mathbf{L}_\diamond$ in \eqref{imp_eq} has a computational complexity of $\mathcal{O}([\prod_{p=1}^{P} N_p]^3)$ in the general case. However, leveraging the separable structure of PGs significantly reduces the complexity to $\mathcal{O}(\sum_{p=1}^{P} N_p^3)$ in \eqref{eff_imp}, since EVD is performed independently on each factor graph. 
Furthermore, if we retain only the top $K_p$ most significant eigenvalue-eigenvector pairs \cite{behmanesh2023tide} per factor graph, the cost of spectral decomposition is further reduced to $\mathcal{O}(N_p^2 K_p)$ for each factor, resulting in an overall complexity of $\mathcal{O}(\sum_{p=1}^{P} N_p^2 K_p)$.
\end{remark}
\subsection{Stability Analysis}
\label{Stability_Analysis}
Inspired by \cite{song2022robustness},
an additive perturbation model on the $p$-th graph is defined by adding an bounded error matrix $\mathbf{E}_p$ (w.r.t. a matrix norm $\vertiii{\cdot}$) to its Laplacian $\mathbf{L}_p$ as:
\begin{equation}
\label{factor_perturb_bounds}
\tilde{\mathbf{L}}_p=\mathbf{L}_p+\mathbf{E}_p;\:\:\:\vertiii{\mathbf{E}_p}\le\varepsilon_p.
\end{equation}
Now, let $\varphi(u,t)$ and $\tilde{\varphi}(u,t)$ denote the outputs of our model under the original and perturbed Laplacian, respectively.
The Laplacian $\mathbf{L}_\diamond$ governs the evolution of ${\underline{\mathbf{U}}_t}_{(P+1)}$ through the flow PDE $\frac{\partial^2 \varphi(u,t)}{\partial t^2} = -\mathbf{L}_\diamond^2 \varphi(u,t)$ as described in \eqref{imp_eq}, with the corresponding solution $\varphi(u,t) = \cos(t \mathbf{L}_\diamond) \varphi(u,0)$, assuming $\left. \frac{\partial\varphi(u,t)}{\partial t}\right|_{t=0}=0$.
The following theorem establishes that, for Cartesian PGs, the error bound on the stability of our model admits the integrated sum of stability over the factor graphs.
\begin{theorem}
\label{thm_e1e2}
Consider the perturbation model in \eqref{factor_perturb_bounds} for $p=1,\hdots,P$ with factor-wise error bounds $\{\vertiii{\mathbf{E}_p} \le \varepsilon_p\}_{p=1}^P$. Let a Cartesian PG (assumed to have no isolated nodes) with true and perturbed Laplacians $\mathbf{L}_\diamond=\oplus_{p=1}^{P}{\mathbf{L}_p}$ and $\tilde{\mathbf{L}}_\diamond=\oplus_{p=1}^{P}{\tilde{\mathbf{L}}_p}$, respectively.
Then, the stability on the true and perturbed outputs 
$\varphi(u,t)$ and $\tilde{\varphi}(u,t)$, respectively, is given by the sum of the individual stability on each factor graph:
\begin{equation}
\vertiii{\varphi(u,t)-\tilde{\varphi}(u,t)}=\sum_{p=1}^{P}{\mathcal{O}(\varepsilon_p)}.
\end{equation}
\end{theorem}
Theorem \ref{thm_e1e2} states that the solution of the \method~\eqref{MDPDE_sol}  remains stable w.r.t. the magnitude of perturbations in the factor graphs. This robustness is particularly important when the factor graphs are subject to errors.

\begin{table*}[!t]
\caption{Traffic forecasting results on \texttt{MetrLA} dataset (\textbf{bold} and \underline{underlined} denote the best and second best, respectively).}
\vskip -6mm
\label{Table_SimpleBaseline}
\begin{center}
\resizebox{.9\textwidth}{!}{
\begin{tabular}{rccccccccr}
\toprule
\multirow{2}{*}{Method}&
\multicolumn{3}{c}{$H=3$}&
\multicolumn{3}{c}{$H=6$}&
\multicolumn{3}{c}{$H=12$}\\
\cmidrule(l){2-4}
\cmidrule(l){5-7}
\cmidrule(l){8-10}

& MAE& MAPE & RMSE &

MAE& MAPE & RMSE &

MAE& MAPE & RMSE\\ 

\midrule

ARIMA \cite{li2018diffusion}&

3.99&9.60\%&8.21&

5.15 & 12.70\% & 10.45 &

6.90 & 17.40\% & 13.23 \\

G-VARMA~\cite{isufi2019forecasting}                 
& 3.60 & 9.62\% & 6.89
& 4.05 & 11.22\% & 7.84
& 5.12 & 14.00\% & 9.58  \\

FC-LSTM~\cite{li2018diffusion}   
& 3.44 & 9.60\% & 6.30 
& 3.77 & 10.90\% & 7.23 
& 4.37 & 13.20\% & 8.69 \\

Graph Wavenet~\cite{wu2019graph}           
& 2.69 &  {6.90\%} & {5.15} 

&  {3.07} &  {8.37\%} &  {6.22} 

& {3.53} & 10.01\% &  {7.37} \\



GGRNN~\cite{ruiz2020gated}                  
& 2.73 & 7.12\% & 5.44 
& 3.31 & 8.97\% & 6.63 
& 3.88 & 10.59\% & 8.14          \\

GRUGCN \cite{gao2022equivalence} 
& 2.69 & \textbf{6.61\%} & {5.15}
& {3.05} & {7.96\%} & {6.04} 
&  {3.62} & {9.92\%} & {7.33} \\

SGP \cite{cini2023scalable} 
& 3.06 & 7.31\% & 5.49
& {3.43} & {8.54\%} & {6.47} 
& 4.03 & 10.53\% & 7.81 \\

GTCNN \cite{sabbaqi2023graph} 
& \underline{2.68} & 6.85\% &  {5.17}
& {3.02} & {8.30\%} & {6.20} 
& {3.55} & {10.21\%} & 7.35 \\

CITRUS \cite{einizade2024continuous}&

2.70&6.74\%&\underline{5.14}&

\underline{2.98}&\underline{7.78}\%&\underline{5.90}&

\underline{3.44}&\textbf{9.28}\%&\underline{6.85}\\

\midrule

\textbf{\method} (Ours)&\textbf{2.66}&\underline{6.63}\%&\textbf{5.09}&\textbf{2.96}&\textbf{7.72}\%&\textbf{5.87}&\textbf{3.38}&\underline{9.29}\%&\textbf{6.84}\\

\bottomrule

\end{tabular}
}
\end{center}
\vskip -7mm
\end{table*}

\subsection{Over-smoothing Analysis}
\label{Oversmoothing}
Over-smoothing in GNNs is commonly characterized by the decay of Dirichlet energy as the number of layers increases, \ie making the graph signals as smooth as possible on the underlying graph \cite{ortega2018graph}. This phenomenon has been theoretically studied in prior works \cite{giovanni2023understanding,rusch2023survey,han2024continuous} and can be formulated for the output of a continuous GNN \cite{han2024continuous}. Specifically, consider the GNN output $\mathbf{U}_t = [\mathbf{u}(t)_1, \ldots, \mathbf{u}(t)_N]^\top$, where $\mathbf{u}(t)_i$ represents the feature vector of node $i$ with degree $\text{deg}_i$. The over-smoothing can then be defined \cite{Oono2020Graph} as $\lim_{t\rightarrow\infty}{E(\mathbf{U}_t)}\rightarrow 0$, with the Dirichlet energy as:
\begin{equation}
\label{e_dir}
\begin{split}
E(\mathbf{U})\vcentcolon=\frac{1}{2}\sum_{(i,j)\in\mathcal{E}}{\left\|\frac{\mathbf{u}_i}{\sqrt{\text{deg}_i}}-\frac{\mathbf{u}_j}{\sqrt{\text{deg}_j}}\right\|^2}=\tr(\mathbf{U}^\top\hat{\mathbf{L}}\mathbf{U}).
\end{split}
\end{equation}
Motivated by \eqref{e_dir}, 
we proceed with a formal analysis of over-smoothing in our model. For the $l$-th layer, consisting of $H_l$ hidden MLP layers, nonlinear activations $\sigma(\cdot)$, and learnable weight matrices $\{\mathbf{W}_{lh}\}_{h=1}^{H_l}$, we define:
\begin{equation}
\label{defs}
\begin{split}
&\mathbf{X}_{l+1}\vcentcolon=f_l(\mathbf{X}_{l}),\:\:\:f_l(\mathbf{X})\vcentcolon=\text{MLP}_l(\cos(t\hat{\mathbf{L}}_{\diamond})\mathbf{X}),\\
&\text{MLP}_l(\mathbf{X})\vcentcolon=\sigma(\hdots\sigma(\sigma(\mathbf{X})\mathbf{W}_{l1})\mathbf{W}_{l2}\hdots\mathbf{W}_{lH_l}),    
\end{split}
\end{equation}
Here, $\mathbf{X}_{l}:=[(\underline{\mathbf{U}}_{l})_{(P+1)}]^\top$ in Proposition \ref{prop:MDPDE_theorem}, and $f_l(\cdot)$ represents a generalized form of \eqref{imp_eq}. In the term $cos(t\hat{\mathbf{L}}_\diamond)$, $\hat{\mathbf{L}}_\diamond$ is defined as the normalized PG Laplacian $\hat{\mathbf{L}}_\diamond \coloneqq \frac{1}{P} \oplus_{p=1}^{P} \hat{\mathbf{L}}_p$, implying a useful spectral property that the spectrum of $\hat{\mathbf{L}}$ spans the interval of $[0,2]$ \cite{einizade2024continuous}. 
The following theorem provides the over-smoothing analysis.
\begin{theorem}
\label{cos_ener_prod}
With $E(.)$ and $\sigma(\cdot)$ being the Dirichlet energy and ReLU or Leaky ReLU in \eqref{defs}, respectively, we have: 
\begin{equation}
\begin{split}
&E(\mathbf{X}_{l})\le (s\cos^2{(t\lambda^\diamond_\phi)})^l E(\mathbf{X}),\\
\end{split}
\end{equation}
where\,\,
{\small\begin{equation}
\cos{(t\lambda^\diamond_\phi)}=\sum_{\substack{S \subseteq \{1,\dots,P\}\\ |S| \ \mathrm{even}}}(-1)^{|S|/2}\prod_{p \in S} \sin(t\lambda^{(p)}_{i_p})\prod_{p \notin S} \cos(t\lambda^{(p)}_{i_p}).
\end{equation}}
Here, $s\vcentcolon=\sup_{l\in\mathbb{N}_{+}}{s_l}$ and $s_l\vcentcolon=\prod_{h=1}^{H_l}{s_{lh}}$ with $s_{lh}$ being the square of maximum singular value of $\mathbf{W}^\top_{lh}$, $\cos^2{(t\lambda^\diamond_\phi)}=\max_{i}{\cos^2{(t\lambda^{\diamond}_i)}}$ and $\lambda^{\diamond}_{\phi}=\sum_{p=1}^{P}{\lambda^{(p)}_{i_p}}$, where $\lambda^{\diamond}_i$ and $\lambda^{(p)}_{i}$ are the $i$-th non-zero eigenvalues of $\hat{\mathbf{L}}_\diamond$ and $\hat{\mathbf{L}}_p$, respectively.
\end{theorem}
\begin{corollary}
\label{corrl}
If $s\cos^2{(t\lambda^\diamond_\phi)}<1$ and $l\rightarrow\infty$, \(E(\mathbf{X}_{l})\) exponentially converges to \(0\), leading to:
{\small
\begin{equation} 
\label{eq:corl}
t \in \bigcup_{k \in \mathbb{Z}} \frac{1}{2\lambda^\diamond_\phi} ( 2k\pi + \arccos(\frac{2}{s}-1), 2(k+1)\pi - \arccos(\frac{2}{s}-1)).
\end{equation}
}
\end{corollary}
Theorem \ref{cos_ener_prod} and Corollary \ref{corrl} demonstrate that the oversmoothing rate is influenced by three key factors: the learnable weight matrices (through $s$), the spectral characteristics of the factor and product graph (via $\lambda^\diamond_\phi$), and the graph receptive field $t$. Unlike first-order TPDEGs \cite{einizade2024continuous}, which yield only a single lower bound on $t$ and consequently suppress high-frequency components,  \eqref{eq:corl} produces infinitely valid intervals. This flexibility enables control over oversmoothing by tuning $t$ and naturally leads to oscillatory solutions acting as an inherent band-pass filter on the factor graphs.

\section{Experimental results}

We assess the performance of \method~on a real-world spatiotemporal traffic forecasting task.
Particularly, we evaluate our method on the \texttt{MetrLA} dataset \cite{li2018diffusion}, containing four months of traffic recordings collected from $207$ highways in Los Angeles County, sampled every $5$ minutes. The spatial graph is constructed with Gaussian kernels based on pairwise node distances \cite{Cini_Torch_Spatiotemporal_2022}, while the temporal dimension is modeled using simple path graphs. The forecasting task involves using the last $30$ minutes of traffic observations to predict the next $15$, $30$, and $60$ minutes, corresponding to horizons $H=3,\,6\,\text{and}\,12$, respectively.

Our model begins with a linear encoder and then applies three stacked \method~blocks, each containing a 3-layer MLP with $F^{\text{MLP}}_l=64$. The resulting node embeddings are concatenated with the original spatiotemporal input and passed through a linear decoder to generate the multi-horizon forecasts. 
Training is performed in $300$ epochs using the mean absolute error (MAE) as the objective, and the best model on the validation set is evaluated on the test set.
The data preprocessing and train–validation–test follow the protocol in \cite{sabbaqi2023graph,einizade2024continuous}. Hyperparameters are set on the validation set with the Adam optimizer.

The performance results (in terms of MAE, MAPE, and RMSE \cite{sabbaqi2023graph,einizade2024continuous}) in Table \ref{Table_SimpleBaseline}, reveal that GNN–based models generally surpass traditional (GSP) methods such as ARIMA \cite{li2018diffusion} and G-VARMA \cite{isufi2019forecasting}, as well as non-graph methods like FC-LSTM \cite{li2018diffusion}. Most notably, \method $\,$ achieves state-of-the-art performance with more significant improvements at longer forecasting horizons ($H>3$), especially compared to GTCNN \cite{sabbaqi2023graph} and CITRUS \cite{einizade2024continuous}.

\section{Conclusion}
We introduced \method, a second-order tensorial PDE framework designed to process data over multiple interacting (product) graphs,
while providing a rigorous principled analysis of stability and over-smoothing. Compared to first-order TPDEG models, \method~does not dampen the medium to high-frequency components
with great importance to periodic or fine-grained applications, like traffic forecasting \cite{li2018diffusion}.



\bibliographystyle{IEEEbib}
\bibliography{refs}

\begin{thebibliography}{10}

\bibitem{kolda2009tensor}
T.~G. Kolda et~al.,
\newblock ``Tensor decompositions and applications,''
\newblock {\em SIAM Review}, vol. 51, no. 3, pp. 455--500, 2009.

\bibitem{song2023brain}
X.~Song et~al.,
\newblock ``Brain network analysis of schizophrenia patients based on hypergraph signal processing,''
\newblock {\em IEEE T-IP}, vol. 32, pp. 4964--4976, 2023.

\bibitem{einizade2023iterative}
A.~Einizade et~al.,
\newblock ``Iterative pseudo-sparse partial least square and its higher order variant: Application to inference from high-dimensional biosignals,''
\newblock {\em IEEE T-CDS}, vol. 16, no. 1, pp. 296--307, 2023.

\bibitem{ortiz2019sparse}
G.~Ortiz-Jiménez et~al.,
\newblock ``Sparse sampling for inverse problems with tensors,''
\newblock {\em IEEE T-SP}, vol. 67, no. 12, pp. 3272--3286, 2019.

\bibitem{einizade2024continuous}
A.~Einizade et~al.,
\newblock ``Continuous product graph neural networks,''
\newblock in {\em NeurIPS}, 2024.

\bibitem{cini2023graph}
A.~Cini et~al.,
\newblock ``Graph deep learning for time series forecasting,''
\newblock {\em ACM Computing Surveys}, vol. 57, no. 12, pp. 1--34, 2025.

\bibitem{kadambari2021product}
S.~K. Kadambari et~al.,
\newblock ``Product graph learning from multi-domain data with sparsity and rank constraints,''
\newblock {\em IEEE T-SP}, vol. 69, pp. 5665--5680, 2021.

\bibitem{monti2017geometric}
F.~Monti et~al.,
\newblock ``Geometric matrix completion with recurrent multi-graph neural networks,''
\newblock in {\em NeurIPS}, 2017.

\bibitem{ortega2018graph}
A.~Ortega et~al.,
\newblock ``Graph signal processing: Overview, challenges, and applications,''
\newblock {\em Proceedings of the IEEE}, vol. 106, no. 5, pp. 808--828, 2018.

\bibitem{einizade2023learning}
A.~Einizade et~al.,
\newblock ``Learning product graphs from spectral templates,''
\newblock {\em IEEE T-SIPN}, vol. 9, pp. 357--372, 2023.

\bibitem{sabbaqi2023graph}
M.~Sabbaqi et~al.,
\newblock ``Graph-time convolutional neural networks: Architecture and theoretical analysis,''
\newblock {\em IEEE T-PAMI}, vol. 45, no. 12, pp. 14625--14638, 2023.

\bibitem{castro2024gegenbauer}
J.~A. Castro-Correa et~al.,
\newblock ``Gegenbauer graph neural networks for time-varying signal reconstruction,''
\newblock {\em IEEE TNNLS}, vol. 35, no. 9, pp. 11734--11745, 2024.

\bibitem{einizade2023productgraphsleepnet}
A.~Einizade et~al.,
\newblock ``{ProductGraphSleepNet}: Sleep staging using product spatio-temporal graph learning with attentive temporal aggregation,''
\newblock {\em Neural Networks}, vol. 164, pp. 667--680, 2023.

\bibitem{behmanesh2023tide}
M.~Behmanesh et~al.,
\newblock ``{TIDE}: Time derivative diffusion for deep learning on graphs,''
\newblock in {\em ICML}, 2023.

\bibitem{han2024continuous}
A.~Han et~al.,
\newblock ``From continuous dynamics to graph neural networks: Neural diffusion and beyond,''
\newblock {\em TMLR}, 2024.

\bibitem{cai2020note}
C.~Cai et~al.,
\newblock ``A note on over-smoothing for graph neural networks,''
\newblock in {\em ICML - Workshop}, 2020.

\bibitem{giraldo2023trade}
J.~H. Giraldo et~al.,
\newblock ``On the trade-off between over-smoothing and over-squashing in deep graph neural networks,''
\newblock in {\em ACM CIKM}, 2023.

\bibitem{xhonneux2020continuous}
L.~Xhonneux et~al.,
\newblock ``Continuous graph neural networks,''
\newblock in {\em ICML}, 2020.

\bibitem{yue2025graph}
J.~Yue et~al.,
\newblock ``Graph wave networks,''
\newblock in {\em Proceedings of the ACM on Web Conference}, 2025.

\bibitem{yue2025hyperbolicpde}
J.~Yue et~al.,
\newblock ``Hyperbolic-{PDE} {GNN}: Spectral graph neural networks in the perspective of a system of hyperbolic partial differential equations,''
\newblock in {\em ICML}, 2025.

\bibitem{rusch2022graph}
T.~K. Rusch et~al.,
\newblock ``Graph-coupled oscillator networks,''
\newblock in {\em ICML}, 2022.

\bibitem{sandryhaila2014big}
A.~Sandryhaila et~al.,
\newblock ``Big data analysis with signal processing on graphs: Representation and processing of massive data sets with irregular structure,''
\newblock {\em IEEE SPM}, vol. 31, no. 5, pp. 80--90, 2014.

\bibitem{song2022robustness}
Y.~Song et~al.,
\newblock ``On the robustness of graph neural diffusion to topology perturbations,''
\newblock in {\em NeurIPS}, 2022.

\bibitem{li2018diffusion}
Y.~Li et~al.,
\newblock ``Diffusion convolutional recurrent neural network: Data-driven traffic forecasting,''
\newblock in {\em ICLR}, 2018.

\bibitem{isufi2019forecasting}
E.~Isufi et~al.,
\newblock ``Forecasting time series with {VARMA} recursions on graphs,''
\newblock {\em IEEE T-SP}, vol. 67, no. 18, pp. 4870--4885, 2019.

\bibitem{wu2019graph}
Z.~Wu et~al.,
\newblock ``Graph {WaveNet} for deep spatial-temporal graph modeling,''
\newblock in {\em IJCAI}, 2019.

\bibitem{ruiz2020gated}
L.~Ruiz et~al.,
\newblock ``Gated graph recurrent neural networks,''
\newblock {\em IEEE T-SP}, vol. 68, pp. 6303--6318, 2020.

\bibitem{gao2022equivalence}
J.~Gao et~al.,
\newblock ``On the equivalence between temporal and static equivariant graph representations,''
\newblock in {\em ICML}, 2022.

\bibitem{cini2023scalable}
A.~Cini et~al.,
\newblock ``Scalable spatiotemporal graph neural networks,''
\newblock in {\em AAAI}, 2023.

\bibitem{giovanni2023understanding}
F.~Di Giovanni et~al.,
\newblock ``Understanding convolution on graphs via energies,''
\newblock {\em TMLR}, 2023.

\bibitem{rusch2023survey}
T.~K. Rusch et~al.,
\newblock ``A survey on oversmoothing in graph neural networks,''
\newblock {\em SAM Research Report}, 2023.

\bibitem{Oono2020Graph}
K.~Oono and T.~Suzuki,
\newblock ``Graph neural networks exponentially lose expressive power for node classification,''
\newblock in {\em ICLR}, 2020.

\bibitem{Cini_Torch_Spatiotemporal_2022}
A.~Cini et~al.,
\newblock ``{Torch Spatiotemporal},'' 3 2022.

\end{thebibliography}

\newpage
\appendix
\onecolumn

\section*{A. Proof Theorem \ref{MDPDE_thm}} First, we state and prove the following Lemma:

\begin{lemma}
\label{lemma:Cosine_prod}
By defining \(g_t^{(0)}(\mathbf{L}_p):=\cos(t\mathbf{L}_p)\) and \(g_t^{(1)}(\mathbf{L}_p):=\sin(t\mathbf{L}_p)\) for a specific $p$, the cosine filter applied on the PG can be stated in terms of the factor graphs as follows:
\begin{equation}
\cos(t\mathbf{L}_\diamond)=\sum_{\substack{\boldsymbol\epsilon\in\{0,1\}^P\\|\boldsymbol\epsilon|\ \text{even}}}
(-1)^{|\boldsymbol\epsilon|/2}\;\downarrow\otimes_{p=1}^P g_t^{(\epsilon_p)}(\mathbf{L}_p),
\end{equation}
\end{lemma}

\begin{proof}
Let $\mathbf{L}_\diamond=\downarrow\oplus_{p=1}^P \mathbf{L}_p$. Using the Kronecker-sum property of the matrix exponential,
\begin{equation}
e^{\mathrm{i}t\mathbf{L}_\diamond}=\downarrow\otimes_{p=1}^P e^{\mathrm{i}t\mathbf{L}_p}.
\end{equation}
Since $\cos(t\mathbf{L}_\diamond)=\Re\big(e^{\mathrm{i}t\mathbf{L}_\diamond}\big)$, where $\Re(.)$ computes the real part of a complex number, it suffices to expand the right-hand side. For each $p$, one can write:
\begin{equation}
e^{\mathrm{i}t\mathbf{L}_i}=\cos(t\mathbf{L}_i)+\mathrm{i}\sin(t\mathbf{L}_i)=g_t^{(0)}(\mathbf{L}_p)+\mathrm{i}\,g_t^{(1)}(\mathbf{L}_p),
\end{equation}
with \(g_t^{(0)}(\mathbf{L}_p):=\cos(t\mathbf{L}_p)\) and \(g_t^{(1)}(\mathbf{L}_p):=\sin(t\mathbf{L}_p)\). The Kronecker product of these sums expands as:
\begin{equation}
\downarrow\otimes_{p=1}^P\big(g_t^{(0)}(\mathbf{L}_p)+\mathrm{i}\,g_t^{(1)}(\mathbf{L}_p)\big)
= \sum_{\boldsymbol\epsilon\in\{0,1\}^P} \mathrm{i}^{|\boldsymbol\epsilon|}\;
\downarrow\otimes_{p=1}^P g_t^{(\epsilon_p)}((\mathbf{L}_p)),    
\end{equation}
because choosing \(\epsilon_p=1\) picks the \(\mathrm{i}\,g_t^{(1)}(\mathbf{L}_p)\) term in factor \(p\), contributing one factor of \(\mathrm{p}\) per chosen sine.

\noindent Taking real parts, only multi-indices with even \(|\boldsymbol\epsilon|=\sum_p\epsilon_p\) remain, and for \(|\boldsymbol\epsilon|=2k\) we have \(\Re(\mathrm{i}^{2k})=\Re((-1)^k)=(-1)^k\). Hence:
\begin{equation}
\cos(t\mathbf{L}_\diamond)=\sum_{\substack{\boldsymbol\epsilon\in\{0,1\}^P\\|\boldsymbol\epsilon|\ \text{even}}}
(-1)^{|\boldsymbol\epsilon|/2}\;\downarrow\otimes_{p=1}^P g_t^{(\epsilon_p)}(\mathbf{L}_p),
\end{equation}
which is the claimed formula.
\end{proof}
Now, using Lemma \ref{lemma:Cosine_prod} and by performing vectorization on two sides of the Eqn. \eqref{MDPDE_sol}:
\begin{equation}
\begin{split}
\mathrm{vec}(\underline{\mathbf{U}}_t)&=\sum_{\substack{\boldsymbol\epsilon\in\{0,1\}^P\\|\boldsymbol\epsilon|\ \text{even}}}
(-1)^{|\boldsymbol\epsilon|/2}\downarrow\otimes_{p=1}^P g_t^{(\epsilon_p)}(\mathbf{L}_p)\mathrm{vec}(\underline{\mathbf{U}}_0)\\
&=\cos(t\mathbf{L}_\diamond)\mathrm{vec}(\underline{\mathbf{U}}_0),
\end{split}
\end{equation}
By defining $\mathbf{u}_t:=\mathrm{vec}(\underline{\mathbf{U}}_t)$ and $\mathbf{u}_0:=\mathrm{vec}(\underline{\mathbf{U}}_0)$ and considering the Cartesian product Laplacian:
\begin{equation}
\mathbf{L}_\diamond = \sum_{p=1}^{P} \mathbf{I}_{N_1} \otimes \cdots \otimes \mathbf{I}_{N_{p-1}} \otimes \mathbf{L}_p \otimes \mathbf{I}_{N_{p+1}} \otimes \cdots \otimes \mathbf{I}_{N_P}.
\end{equation}
So, we have $\mathbf{u}_t = \cos(t \mathbf{L}_\diamond) \mathbf{u}_0$, and the following steps are needed to conclude the proof:

Using the standard derivative formula for the matrix cosine $\frac{d^2}{dt^2} \cos(t \mathbf{L}_\diamond) = - \mathbf{L}_\diamond^2 \cos(t \mathbf{L}_\diamond)$, we have $\mathbf{u}_t'' = - \mathbf{L}_\diamond^2 \mathbf{u}_t$. Define the Kronecker sum terms:
\begin{equation}
\mathbf{L}_i^{(i)} := \mathbf{I}_{N_1} \otimes \cdots \otimes \mathbf{I}_{N_{i-1}} \otimes \mathbf{L}_i \otimes \mathbf{I}_{N_{i+1}} \otimes \cdots \otimes \mathbf{I}_{N_P}.
\end{equation}
\noindent Then:
\begin{equation}
\begin{split}
&\mathbf{L}_\diamond = \sum_{i=1}^P \mathbf{L}_i^{(i)}, \\
&\mathbf{L}_\diamond^2 = \sum_{i=1}^P (\mathbf{L}_i^{(i)})^2 + 2 \sum_{1 \le i < j \le P} \mathbf{L}_i^{(i)} \mathbf{L}_j^{(j)}.
\end{split}
\end{equation}
\noindent For any tensor $\underline{\mathbf{X}}$:
\begin{equation}
\begin{split}
&(\mathbf{L}_i^{(i)})^2 \, \mathrm{vec}(\underline{\mathbf{X}}) = \mathrm{vec}(\underline{\mathbf{X}} \times_i \mathbf{L}_i^2),\\
&\mathbf{L}_i^{(i)} \mathbf{L}_j^{(j)} \, \mathrm{vec}(\underline{\mathbf{X}}) = \mathrm{vec}(\underline{\mathbf{X}} \times_i \mathbf{L}_i \times_j \mathbf{L}_j).
\end{split}
\end{equation}
Applying this to $\mathbf{u}_t'' = - \mathbf{L}_\diamond^2 \mathbf{u}_t$ and reshaping into tensor form gives:
\begin{equation}
\frac{\partial^2\underline{\mathbf{U}}_t}{\partial\:t^2} = - \sum_{i=1}^{P} \underline{\mathbf{U}}_t \times_i \mathbf{L}_i^2 - 2 \sum_{1 \le i < j \le P} \underline{\mathbf{U}}_t \times_i \mathbf{L}_i \times_j \mathbf{L}_j,
\end{equation}
and the proof is completed.

\section*{B. Proof of Proposition \ref{prop:MDPDE_theorem}} Now, by performing mode-$(P+1)$ unfolding on the tensor $f(\underline{\mathbf{U}})$ in Eq. \eqref{tensor_formm2_2} and relying on Lemma \ref{lemma:Cosine_prod}, we have:
\begin{equation}
\label{tensor_formm2_2}
\begin{split}
{f(\underline{\mathbf{U}}_l)}_{(P+1)}=\sum_{\substack{\boldsymbol\epsilon\in\{0,1\}^P\\|\boldsymbol\epsilon|\ \text{even}}}
(-1)^{|\boldsymbol\epsilon|/2}\mathbf{W}^\top_{l}{\underline{\mathbf{U}}_l}_{(P+1)}\left[\downarrow\otimes_{p=1}^P g_t^{(\epsilon_p)}(\mathbf{L}_p)\right]^\top.
\end{split}
\end{equation}
Now, by applying transposition and further simplification, we get:
\begin{equation}
\label{tensor_formm2_3}
\begin{split}
[{f(\underline{\mathbf{U}}_l)}_{(P+1)}]^\top&=\sum_{\substack{\boldsymbol\epsilon\in\{0,1\}^P\\|\boldsymbol\epsilon|\ \text{even}}}
(-1)^{|\boldsymbol\epsilon|/2}\downarrow\otimes_{p=1}^P g_{t}^{(\epsilon_p)}(\mathbf{L}_p)\;[{\underline{\mathbf{U}}_l}_{(P+1)}]^\top\mathbf{W}_l\\
&=\cos(t\mathbf{L}_\diamond)[{\underline{\mathbf{U}}_l}_{(P+1)}]^\top\mathbf{W}_l,
\end{split}
\end{equation}
which concludes the proof.

\section*{D. Proof of Theorem \ref{thm_e1e2}}
Due to the filtering approach $\varphi(u,t) = \cos(t\mathbf{L}_\diamond) \varphi(u,0)$, we only need to bound $\|\cos(t\tilde{\mathbf{L}}_\diamond) - \cos(t\mathbf{L}_\diamond)\|$, where $\tilde{\mathbf{L}}-\mathbf{L}=\mathbf{E}$. The following Lemma describes the properties of Lipschitz continuity for cosine operators.

\begin{lemma}
\label{lemma:cos_lip}
Let $\mathbf{A},\mathbf{B}$ be Hermitian matrices and let $t\ge 0$. Then:
\[
\|\cos(t\mathbf{A})-\cos(t\mathbf{B})\| \le t\,\|\mathbf{A}-\mathbf{B}\|.
\]
Moreover, the constant $t$ is optimal: one cannot replace $t$ by a smaller constant valid for all Hermitian $\mathbf{A},\mathbf{B}$.
\end{lemma}

\begin{proof}
Using $\cos(\mathbf{X})=\tfrac12(e^{\mathrm{i}\mathbf{X}}+e^{-\mathrm{i}\mathbf{X}})$ and the Duhamel identity:
\[
e^{\mathrm{i}t\mathbf{A}}-e^{\mathrm{i}t\mathbf{B}}
= \mathrm{i}\int_0^t e^{\mathrm{i}(t-s)\mathbf{A}}(\mathbf{A}-\mathbf{B})e^{\mathrm{i}s\mathbf{B}}\,ds,
\]
we obtain, since $e^{\mathrm{i} \mathbf{A}}$ and $e^{\mathrm{i} \mathbf{B}}$ are unitary,
\[
\|e^{\mathrm{i}t\mathbf{A}}-e^{\mathrm{i}t\mathbf{B}}\| \le \int_0^t \|\mathbf{A}-\mathbf{B}\|\,ds = t\|\mathbf{A}-\mathbf{B}\|.
\]
The same bound holds for $e^{-\mathrm{i}t\mathbf{A}}-e^{-\mathrm{i}t\mathbf{B}}$. Averaging yields:
\[
\begin{split}
&\|\cos(t\mathbf{A})-\cos(t\mathbf{B})\|\\
&\le \tfrac12\|e^{\mathrm{i}t\mathbf{A}}-e^{\mathrm{i}t\mathbf{B}}\| + \tfrac12\|e^{-\mathrm{i}t\mathbf{A}}-e^{-\mathrm{i}t\mathbf{B}}\|\le t\|\mathbf{A}-\mathbf{B}\|.
\end{split}
\]
For optimality, note that restricting to commuting (diagonal) matrices reduces the inequality to the scalar Lipschitz bound $|\cos(t\lambda)-\cos(t\mu)|\le C|\lambda-\mu|$, which forces $C\ge\sup_x|\frac{d\,\cos(tx)}{d\,x}|=\sup_x|(-t\sin(tx))|=t$.
\end{proof}
Then, we need another Proposition:
\begin{proposition}
\label{prop_prodE}
(from \cite{einizade2024continuous}) Let $\mathbf{L}_\diamond$ and $\tilde{\mathbf{L}}_\diamond$ denote the original and perturbed Laplacian matrices of the Cartesian PG, where the perturbation for each factor graph follows the model described in \eqref{factor_perturb_bounds}. Under this formulation, the perturbation matrix $\mathbf{E}$ admits the Cartesian structure, such that $\mathbf{E} = \oplus_{p=1}^{P} \mathbf{E}_p$, and the following relationship holds:
\begin{equation}
\label{factor_perturb_bounds2}
\tilde{\mathbf{L}}_\diamond=\mathbf{L}_\diamond+\mathbf{E};\:\:\:\vertiii{\mathbf{E}}\le\sum_{p=1}^{P}{\varepsilon_p},
\end{equation}
\end{proposition}

Therefore, using Lemma \ref{lemma:cos_lip} and Proposition \ref{prop_prodE}, one can write:
\begin{equation}
\begin{split}
&\vertiii{\varphi(u,t)-\tilde{\varphi}(u,t)}\le t\,.\vertiii{\varphi(u,0)}.\,\vertiii{\mathbf{E}}\\
&\le t\,\vertiii{\varphi(u,0)}\,\left(\sum_{p=1}^{P}{\varepsilon_p}\right)=\sum_{p=1}^{P}{\mathcal{O}(\varepsilon_p)}.
\end{split}
\end{equation}

\section*{E. Proof of Theorem \ref{cos_ener_prod}}
First, Note that Using \(\tilde{\mathbf{x}}\) as the Graph Fourier Transform (GFT) \cite{ortega2018graph} of \(\mathbf{x}\) w.r.t. \(\hat{\mathbf{L}}\) (with eigenvalues $\{\lambda_i\}_{i=1}^N$), one can write \cite{cai2020note}:
\begin{equation}
E(\mathbf{x})=\mathbf{x}^\top \hat{\mathbf{L}}\mathbf{x}=\sum_{i=1}^{N}{\lambda_i\tilde{x}^2_i}.
\end{equation}
Next, by defining $\lambda$ as the smallest non-zero eigenvalue of the Laplacian $\hat{\mathbf{L}}$, the following lemma characterizes the over-smoothing aspects of applying a heat kernel in the simplest case.
\begin{lemma}
\label{exp_ener}
We have:
\begin{equation}
E(\cos{(t\hat{\mathbf{L}})}\mathbf{x})\le \cos^2{(t\lambda_\phi)}E(\mathbf{x}),
\end{equation}
where $\cos^2{(t\lambda_\phi)}=\max_{i}{\cos^2{(t\lambda_i)}}$.
\end{lemma}
\begin{proof} 
By considering the EVD forms of $\hat{\mathbf{L}}=\mathbf{V} \boldsymbol{\Lambda}\mathbf{V}^\top$ and $\cos{(t\hat{\mathbf{L}})}=\mathbf{V} \cos{(t\boldsymbol{\Lambda})}\mathbf{V}^\top$, one can write:
\begin{equation}
\begin{split}
&E(\cos{(t\hat{\mathbf{L}})}\mathbf{x})\\
&=\mathbf{x}^\top \overbrace{{\cos{(t\hat{\mathbf{L}})}}^\top}^{\mathbf{V} \cos{(t\boldsymbol{\Lambda})}\mathbf{V}^\top} \overbrace{\hat{\mathbf{L}}}^{\mathbf{V} \boldsymbol{\Lambda}\mathbf{V}^\top} \overbrace{\cos{(t\hat{\mathbf{L}})}}^{\mathbf{V} \cos{(t\hat{\boldsymbol{\Lambda}})}\mathbf{V}^\top}\mathbf{x}=\sum_{i=1}^{N}{\lambda_i\tilde{x}^2_i \cos^2{(t\lambda_i)}}\\
&\le \cos^2{(t\lambda_\phi)}\left(\sum_{i=1}^{N}{\lambda_i\tilde{x}^2_i}\right)=\cos^2{(t\lambda_\phi)}E(\mathbf{x}).
\end{split}
\end{equation}
Note that in the above proof, we ruled out the zero eigenvalues since they are useless in analyzing the Dirichlet energy.   
\end{proof}
Then by considering the following lemmas from \cite{cai2020note}:
\begin{lemma}
(Lemma 3.2 in \cite{cai2020note}). \(E(\mathbf{X}\mathbf{W})\le\|\mathbf{W}^\top\|_2^2 E(\mathbf{X})\).    
\end{lemma}
\begin{lemma}
(Lemma 3.3 in \cite{cai2020note}). For ReLU and Leaky-ReLU nonlinearities \(E(\sigma(\mathbf{X}))\le E(\mathbf{X})\).    
\end{lemma}
\begin{lemma}[Closed form for a cosine of a sum]\label{thm:main}
For any integer $P\ge 1$ and any real numbers $\lambda_1,\dots,\lambda_P\in\mathbb{R}$,
\begin{equation}
\label{eq:main}
\cos\!\Biggl(\sum_{p=1}^{P} \lambda_p\Biggr)
=
\sum_{\substack{S \subseteq \{1,\dots,P\} \\ |S| \ \mathrm{even}}}
(-1)^{|S|/2}\;
\prod_{p \in S} \sin(\lambda_p)\;
\prod_{p \notin S} \cos(\lambda_p).
\end{equation}
\end{lemma}

\begin{proof}
We set our proof by induction using the two-angle formula. 
We prove \eqref{eq:main} for all $P\in\mathbb{N}$.
For $P=1$, the only even subset is $\varnothing$, giving $\cos(\lambda_1)$.
Assume \eqref{eq:main} holds for some $P$ and let $L_P=\sum_{p=1}^{P}\lambda_p$.
Using $\cos(x+y)=\cos (x)\cos (y)-\sin (x)\sin (y)$,
\begin{equation}\label{eq:induction}
\cos(L_P+\lambda_{P+1})=\cos (L_P)\cos(\lambda_{P+1})-\sin (L_P)\sin(\lambda_{P+1}).
\end{equation}
The companion identity is
\begin{equation}
\sin(L_P)
=
\sum_{\substack{S\subseteq\{1,\dots,P\}\\ |S|\ \mathrm{odd}}}
(-1)^{(|S|-1)/2}
\Bigl(\prod_{p\in S}\sin(\lambda_p)\Bigr)\!
\Bigl(\prod_{p\notin S}\cos(\lambda_p)\Bigr).
\end{equation}
Substituting the two expansions into \eqref{eq:induction} and distributing shows that
all even-cardinality subsets $T\subseteq\{1,\dots,P+1\}$ occur with coefficient $(-1)^{|T|/2}$, 
yielding \eqref{eq:main} with $P+1$ in place of $P$. Thus, the result holds for all $P$ by induction.
\end{proof}
Therefore, by combining the previous concepts, for $l$ layers, it can be said that:
\begin{equation}
\begin{split}
&E(\mathbf{X}_{l})\le (s\cos^2{(t\lambda^\diamond_\phi)})^l E(\mathbf{X}),\\
&\text{where}\\
&\cos{(t\,\lambda^\diamond_\phi)}=\sum_{\substack{S \subseteq \{1,\dots,P\} \\ |S| \ \mathrm{even}}}
(-1)^{|S|/2}\;
\prod_{p \in S} \sin(t\,\lambda^{(p)}_{i_p})\;
\prod_{p \notin S} \cos(t\,\lambda^{(p)}_{i_p}).
\end{split}
\end{equation}
Here, $\lambda^{(p)}_{i_p}$ is the $i_p$-th eigenvalue of the $p$-th normalized factor graph such that $\cos^2{(t\lambda^\diamond_\phi)}=\max_{i}{\cos^2{(t\lambda^{\diamond}_i)}}$, where $\lambda^{\diamond}_i$ is the $i$-th non-zero eigenvalue of $\hat{\mathbf{L}}_\diamond$ and $\lambda^{\diamond}_{\phi}=\sum_{p=1}^{P}{\lambda^{(p)}_{i_p}}$.

So, \(E(\mathbf{X}_{l})\) exponentially converges to \(0\), when:
\begin{equation}
\begin{split}
&\lim_{l\rightarrow\infty}{(s\cos^2{(t\lambda_\phi)})^l}=0 \leftrightarrow s\cos^2{(t\lambda_\phi)}<1,\rightarrow\\
&t \in \bigcup_{k \in \mathbb{Z}} \frac{1}{2\lambda^\diamond_\phi}( 2k\pi + \arccos(\frac{2}{s}-1),\; 2(k+1)\pi - \arccos(\frac{2}{s}-1) ).
\end{split}
\end{equation}

\end{document}